\documentclass[onecolumn,conference,compsoc]{IEEEtran}
\usepackage{microtype}
\usepackage{graphicx}
\usepackage{subfigure}
\usepackage{booktabs}
\usepackage{hyperref}
\usepackage{bm}
\usepackage{amsfonts}
\usepackage{amsmath}
\usepackage{amsthm}
\usepackage[T1]{fontenc}
\usepackage[utf8]{inputenc}
\usepackage{authblk}
\newtheorem{definition}{Definition}

\usepackage{microtype}
\usepackage{amsmath,amssymb,amsfonts,amstext,amsthm,bm}
\usepackage{latexsym,graphicx,epsf,epsfig,psfrag}
\usepackage{subfigure}
\usepackage{dsfont}
\usepackage{booktabs} 
\usepackage{cleveref}
\usepackage{enumitem}
\usepackage{balance}
\usepackage{color}
\usepackage{wrapfig}

\newtheorem{lemma}{Lemma}

\ifCLASSOPTIONcompsoc
  \usepackage[nocompress]{cite}
\else
  \usepackage{cite}
\fi
\ifCLASSINFOpdf
\else
\fi
\begin{document}
%
\title{On the Universal Approximation Property and Equivalence of Stochastic Computing-based Neural Networks and Binary Neural Networks}

\author[1]{Yanzhi Wang}
\author[1]{Zheng Zhan}
\author[1]{Jiayu Li}
\author[1]{Jian Tang}
\author[2]{Bo Yuan}
\author[3]{Liang Zhao}
\author[4]{\\Wujie Wen}
\author[5]{Siyue Wang}
\author[5]{Xue Lin}
\affil[1]{Department of Electrical Engineering and Computer Science, Syracuse University}
\affil[2]{Department of Electrical Engineering, City College of the City University of New York}
\affil[3]{Department of Mathematics and Computer Science, Lehman College of CUNY}
\affil[4]{Department of Electrical and Computer Engineering, Florida International University}
\affil[5]{Department of Electrical and Computer Engineering, Northeastern University}
\renewcommand\Authands{ and }

%


\maketitle

\begin{abstract}
	Large-scale deep neural networks are both memory and computation-intensive, thereby posing stringent requirements on the computing platforms. Hardware accelerations of deep neural networks have been extensively investigated. Specific forms of binary neural networks (BNNs) and stochastic computing-based neural networks (SCNNs) are particularly appealing to hardware implementations since they can be implemented almost entirely with binary operations.

Despite the obvious advantages in hardware implementation, these approximate computing techniques are questioned by researchers in terms of accuracy and universal applicability. Also it is important to understand the relative pros and cons of SCNNs and BNNs in theory and in actual hardware implementations. In order to address these concerns, in this paper we prove that the "ideal" SCNNs and BNNs satisfy the universal approximation property with probability 1 (due to the stochastic behavior). The proof is conducted by first proving the property for SCNNs from the strong law of large numbers, and then using SCNNs as a "bridge" to prove for BNNs. Based on the universal approximation property, we further prove that SCNNs and BNNs exhibit the same energy complexity. In other words, they have the same asymptotic energy consumption with the growing of network size. We also provide a detailed analysis of the pros and cons of SCNNs and BNNs for hardware implementations and conclude that SCNNs are more suitable.

\end{abstract}


%
\IEEEpeerreviewmaketitle

\section{Introduction}
Large-scale neural networks are both memory-intensive and computation-intensive, thereby posing stringent requirements on the computing platforms when deploying those large-scale neural network models on memory-constrained and energy-constrained embedded devices. 
In order to overcome these limitations, the hardware accelerations of deep neural networks have been extensively investigated in both industry and academia \cite{mahajan2016tabla,zhao2017accelerating,umuroglu2017finn,company1,company2,han2016eie,chen2014dadiannao,moons201714}. 
These hardware accelerations are based on FPGA and ASIC devices and can achieve a significant improvement on energy efficiency, along with small form factor, compared with traditional CPU or GPU based computing of deep neural networks. 
Both characteristics are critical for the battery-powered embedded and autonomous systems.

Hardware systems, including FPGAs and ASICs, have much higher peak performance for binary operations compared to floating point ones. Besides, it is also desirable to reduce the model size of deep neural network such that the whole model can be stored using on-chip memory, thereby reducing the timing and energy overheads of off-chip storage and communications. 
As a result, the Binary Neural Networks (BNNs), proposed by \cite{courbariaux2015binaryconnect}, are particularly appealing since they can be implemented almost entirely with binary operations, with the potential to attain performance in the tera-operations per second (TOPS) range on FPGAs or ASICs. 

Besides BNNs, reference work \cite{ren2017sc,yu2017accurate,kyounghoon2015approximate,merolla2014million,liquantized,neftci2016stochastic,andreou2016software} have also proposed to utilize the hardware-oriented Stochastic Computing (SC) technique for developing (large-scale) deep neural networks, i.e., SCNNs. The SC technique represents a number using the portion of 1's in a bit sequence. Many key operations in neural networks, such as multiplications and additions, can be implemented in a single gate in SC. For example, multiplication of two stochastic numbers can be implemented using a single AND gate or XNOR gate (depending on unipolar or bipolar representations).
It enables the efficient implementation of deep neural networks with extremely small hardware footprint.

The BNNs and SCNNs are essentially alike: Both rely on binary operations and very simple calculations in hardware such as AND, XNOR gates, multiplexers and counters. For their distinctions, SCNNs "stretch" in the temporal domain and use a bit sequence (stochastic number) to approximate a real number, whereas BNNs "span" in the spatial domain and require more input and hidden neurons to maintain the desired accuracy. 

Despite the obvious advantages in hardware implementation, these approximate computing techniques are questioned by researchers in terms of accuracy. Will SCNNs and BNNs be accurate for any types of neural networks and applications? More specifically, conventional neural networks with at least one hidden layer satisfy the \emph{universal approximation property} \cite{csaji2001approximation} in that they can approximate an arbitrary continuous or measurable function given enough number of neurons in the hidden layer. Will SCNNs and BNNs satisfy such property as well? Finally, what are the relative pros and cons of SCNNs and BNNs in theory, and at the hardware level?

In this paper we aim to answer the above questions. We consider the "ideal" SCNNs and BNNs that are independent of specific hardware implementations. As the key contribution of this paper, we prove that SCNNs and BNNs satisfy the universal approximation property with probability 1 (due to the stochastic behavior in these networks). The proof is conducted by first proving the property for SCNNs from the strong law of large numbers, and then using SCNNs as a "bridge" to prove for BNNs. This is because it is difficult to directly prove the property for BNNs, as BNNs represent functions with discrete (binary) input values instead of continuous ones.

Based on the universal approximation property, we further prove that SCNNs and BNNs exhibit the same energy complexity. In other words, they have the same asymptotic energy consumption with the growing of network size. We also provide a detailed analysis of the pros and cons of SCNNs and BNNs for hardware implementations and conclude that SCNNs are more suitable for hardware.

\section{Background and Related Work}
\subsection{Stochastic Computing and SCNNs}

Stochastic computing (SC) is a paradigm that represents a number, named \emph{stochastic number}, by counting the number of ones in a bit-stream. For example, the bit-stream 0100110100 contains four ones in a ten-bit stream, thus it represents $x=P(X=1)=4/10=0.4$. In the bit-stream, each bit is independent and identically distributed (i.i.d.) which can be generated in hardware using stochastic number generators (SNGs). Obviously, the length of the bit-streams can significantly affect the calculation accuracy in SC \cite{brown2001stochastic}. 
In addition to this unipolar encoding format, SC can also represent numbers in the range of $[-1,1]$ using the bipolar encoding format. In this scenario, a real number $x$ is processed by $P(X=1)=(x+1)/2$. Thus 0.4 can be represented by 1011011101, as $P(X=1)=(0.4+1)/2=7/10$. 

\begin{wrapfigure}{r}{0.5\textwidth}
	\begin{center}
		\includegraphics[width = 0.5\textwidth]{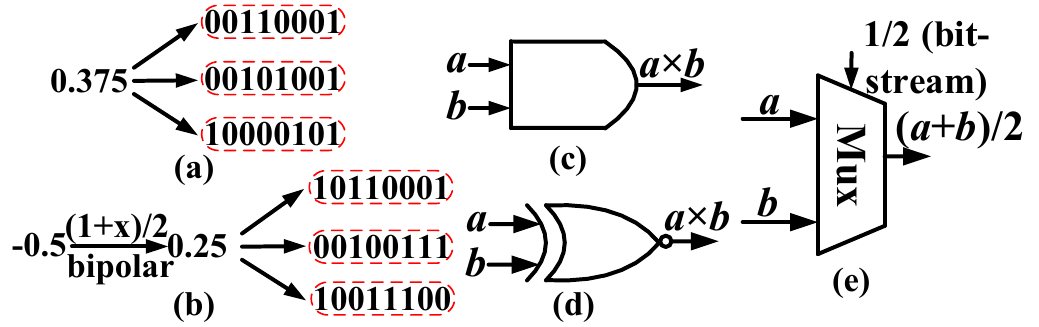}
		\caption{(a) Unipolar encoding format and (b) bipolar encoding format. (c) AND gate for unipolar multiplication. (d) XNOR gate for bipolar multiplication. (e) MUX gate for addition.}
		\label{fig_SCgate}
	\end{center}
	\vspace{-0.3em}
\end{wrapfigure}

Compared to conventional computing, the major advantage of stochastic computing is the significantly lower hardware cost for a large category of arithmetic calculations. 
A summary of the basic computing components in SC, such as multiplication and addition, is shown in Figure \ref{fig_SCgate}. As an illustrative example, a unipolar multiplication can be performed by a single AND gate since $P(A\cdot B=1)=P(A=1)P(B=1)$ (assuming independence), and a bipolar multiplication is performed by a single XNOR gate since $c=2P(C=1)-1=2(P(A=1)P(B=1)+P(A=0)P(B=0))-1=(2P(A=1)-1)(2P(B=1)-1)=ab$. 

Besides multiplications and additions, SC-based activation functions are also developed \cite{li2017deep,li2017hardware}. As a result, SC has become an interesting and promising approach to implement large-scale neural networks  \cite{yuan2017softmax,yu2017accurate,li2017structural,kyounghoon2015approximate} with high performance/energy efficiency and minor accuracy degradation.

\subsection{Binary Neural Networks (BNNs)}

BNNs use binary weights, i.e., weights that are constrained to only two possible values (not necessarily 0 and 1) \cite{courbariaux2015binaryconnect}.
BNNs also have great potential to facilitate consumer applications on low-power devices and embedded systems.
\cite{zhao2017accelerating,umuroglu2017finn} have implemented BNNs in FPGAs with high performance and modest power consumption.

BNNs constrain the weights to either $+1$ or $-1$ during the forward propagation process. 
As a result, many multiply-accumulate operations are replaced by simple additions (and subtractions) using single gates.
This results in a huge gain in hardware resource efficiency, as fixed-point adders/accumulators are much less expensive both in area and energy than fixed-point multiply-accumulators \cite{david2007hardware}.

The real-valued weights are transformed into the two possible values through the following stochastical binarization operation:
\begin{equation}
w_{\text{B}}=\left\{
\begin{aligned}
+1 \qquad & \text{with probability} \quad p=\sigma(w) \\
-1 \qquad & \text{with probability} \quad 1-p \\
\end{aligned}
\right.
\end{equation}
where $\sigma$ is the \emph{hard sigmoid} function:
\begin{equation}
\sigma(x)=\text{clip}(\frac{x+1}{2},0,1)=\max(0,\min(1,\frac{x+1}{2}))
\end{equation}
A hard sigmoid rather than the soft version is used because it is far less computationally expensive.

At training time, BNNs randomly pick one of two values for each weight, for each minibatch, for both the forward and backward propagation phases of backpropagation. However, the stochastic gradient descent (SGD) update is accumulated in a real-valued variable storing the parameter to average the noise for keeping sufficient resolution.
Moreover, binarization process adds some noise into the model, which provides a form of generalization to address the over-fitting problem. 


\subsection{Universal Approximation Property}
For feedforward neural networks with one hidden layer, \cite{cybenko1989approximation} and \cite{hornik1989multilayer} have proved separately the universal approximation property, which guarantees that for any given continuous function or measurable function and any error bound $\epsilon >0$, there always exists a single-hidden layer neural network that approximates the function within $\epsilon$ integrated error. Besides the approximation property itself, it is also desirable to cast a limit on the maximum amount of neurons. In this direction, \cite{barron1993universal} showed that feedforward networks with one layer of sigmoidal nonlinearities achieve an integrated squared error with order of O$(1/n)$, where $n$ is the number of neurons.

More recently, several interesting results were published on the approximation capabilities of deep neural networks or neural networks using structured matrices. \cite{delalleau2011shallow} have shown that there exists certain functions that can be approximated by three-layer neural networks with a polynomial amount of neurons, while two-layer neural networks require exponentially larger amount to achieve the same error. \cite{montufar2014number} and \cite{telgarsky2016benefits} have shown the exponential increase of linear regions as neural networks grow deeper. \cite{liang2016deep} proved that with $\log(1/\epsilon)$ layers, the neural network can achieve the error bound $\epsilon$ for any continuous function with O($polylog(\epsilon)$) parameters in each layer. Recently, \cite{zhao2017theoretical} have proved that neural networks represented in structured, low displacement rank matrices preserve the universal approximation property. These recent research have sparked the research interests on the theoretical properties of neural networks with simplifications/approximations which are suitable for high-efficiency hardware implementations.

\section{Neural Network of Interests and SCNNs}

Our problem statement follows the flow of reference work \cite{zhao2017theoretical} for investigating the universal approximation property. Let $I_{n}$ denote the $n$-dimensional unit cube, $[0,1]^{n}$. The space of continuous functions on $I_{n}$ is denoted by $C(I_{n})$. A feedforward neural network with $N$ units of neurons arranged in a single hidden layer is denoted by a function $G:\mathbb{R}^n\rightarrow\mathbb{R}$, satisfying the form
\begin{equation}
G(\bm{x})=\sum^N_{i=1}\alpha_{i}\sigma(\bm{w}^{\mathsf{T}}_{i}\bm{x}+b_{i})
\end{equation}
where $\bm{w}_{i}$, $\bm{x}\in\mathbb{R}^n$, $\alpha_{i}$, $b_{i}\in\mathbb{R}$, and $\sigma$ is a nonlinear sigmoidal activation function. The $\bm{w}_{i}$ denotes weights associated with hidden neuron $i$ and is applied to input $\bm{x}$. $\alpha_{i}$ denotes the $i$-th weight of output neuron, and is applied to the output of $i$-th neuron in the hidden layer. $b_{i}$
is the bias of unit $i$.

\begin{definition}
	A sigmoidal activation function
	$\sigma:\mathbb{R}\rightarrow\mathbb{R}$ satisfies
	\[ \sigma(t)\rightarrow\begin{cases}
	1 & as\quad t\rightarrow\infty\\
	0 & as\quad t\rightarrow-\infty
	\end{cases} \]
\end{definition}

\begin{definition}
	Starting from the neural network of interests, we define an SCNN satisfying the form:
	\begin{equation}
	G_{\text{SC,M}}(\bm{x}_{\text{SC,M}})=\sum^N_{i=1}\alpha_{i}\sigma(\bm{w}^{\mathsf{T}}_{\text{SC,M}, i}\bm{x}_{\text{SC,M}}+b_{\text{SC,M},i})
	\end{equation}
	where each element $j$ in $\bm{w}^{\mathsf{T}}_{\text{SC,M}, i}$ is denoted by $w^{j}_{\text{SC,M}, i}$, and each element in $\bm{x}_{\text{SC,M}}$ is denoted by $x^j_{\text{SC,M}}$. $w^{j}_{\text{SC,M}, i}$, $x^j_{\text{SC,M}}$, and $b_{\text{SC,M},i}$ 
	are stochastic numbers represented by $M$-bit streams, as approximations of $w^{j}_{i}$, $x^j$, and $b_{i}$, respectively.
	These bit-streams are independent in each bit and $\bm{w}^{\mathsf{T}}_{\text{SC,M}, i}$, $\bm{x}_{\text{SC,M}}$, and $b_{\text{SC,M},i}$ will converge to $\bm{w}^{\mathsf{T}}_i$, $\bm{x}$, and $b_i$ as $M\rightarrow \infty$, respectively.
	The computation in $\bm{w}^{\mathsf{T}}_{\text{SC,M}, i}\bm{x}_{\text{SC,M}}+b_{\text{SC,M},i}$ follows the SC rules described in Section 2.1.
\end{definition}

In the above definitions we focus on an "ideal" SCNN that assumes accurate activation and output layer calculation (which is reasonable because the output layer size is typically very small). The SCNN of interest, as illustrated in Figure \ref{fig_SCNN}, does not depend on specific hardware implementations that may be different in practice. We also do not specify any limitation on the weight and input ranges because they can be effectively dealt with by pre-scaling techniques.

\section{The Universal Approximation Property of SCNNs and BNNs}

\begin{wrapfigure}{r}{0.44\textwidth}
	\includegraphics[width = 0.44\textwidth]{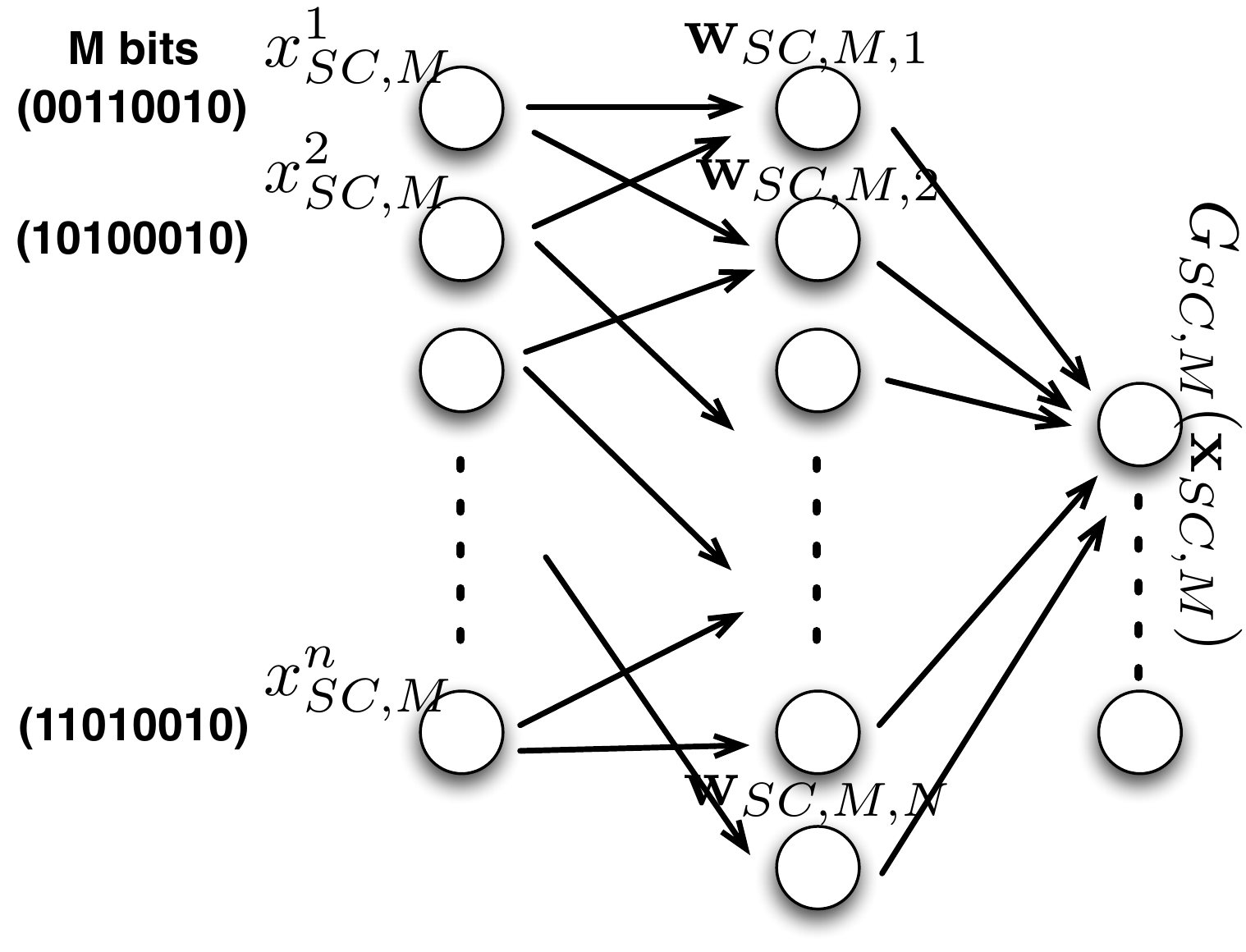}
	\caption{The structure of SCNN of interest.}
	\label{fig_SCNN}
	\vspace{-1.cm}
\end{wrapfigure}

In this section, we prove that SCNNs and BNNs satisfy the universal approximation property with probability 1. More specifically, we first prove the property for SCNNs and then use SCNNs as a "bridge" to prove for BNNs. This two-step proof is due to the fact that directly proving the property for BNNs is difficult, as BNNs represent functions with binary input values.

\subsection{Universal Approximation Property of SCNN}
In this section we will prove a lemma on the closeness of stochastic approximation for the inputs of each neuron, a lemma on the closeness of approximations for the outputs, and finally extend the universal approximation theorem from \cite{cybenko1989approximation} to SCNNs.
\begin{lemma}\label{lem:1}
	As the bit-stream length $M\rightarrow \infty$, the stochastic number
	$\bm{w}^{\mathsf{T}}_{\text{SC},M, i}\bm{x}_{\text{SC},M}+b_{\text{SC},M,i}$ converges to $\bm{w}^{\mathsf{T}}_{ i}\bm{x}+b_{i}$ almost surely.
\end{lemma}
\begin{proof}
	Let $\Omega$ be the sample space of all bit-streams generated to represent elements in $\bm{w}^{\mathsf{T}}_{ i}$, $\bm{x}$, and $b_i$. For each instance $\omega\in\Omega$, use notations $\bm{w}^{\mathsf{T}}_{\text{SC},M, i}(\omega)$, $\bm{x}_{\text{SC},M}(\omega)$, and $b_{\text{SC},M,i}(\omega)$ to represent stochastic numbers (or vectors) calculated from the corresponding $M$-bit streams associated with $\omega$. Moreover, define three constant random variables representing the target real values, namely for each $i\in\{1,...,N\}$,
	\begin{equation}
	\left\{
	\begin{array}{c}
	\bm{w}^{\mathsf{T}}_{ i}(\omega)\equiv \bm{w}^{\mathsf{T}}_{ i}, \forall\omega\in\Omega,\\
	\bm{x}(\omega)\equiv \bm{x}, \forall\omega\in\Omega,\\
	b_i(\omega)\equiv b_i, \forall\omega\in\Omega,
	\end{array}
	\right.
	\end{equation}
	We shall prove that for every $\omega\in\Omega$:
	\begin{equation}
	\lim_{M\rightarrow\infty}\bm{w}^{\mathsf{T}}_{\text{SC},M, i}(\omega)\cdot\bm{x}_{\text{SC},M}(\omega) + b_{\text{SC},M,i}(\omega) = \bm{w}^{\mathsf{T}}_{ i}(\omega)\cdot\bm{x}(\omega) + b_i(\omega).
	\end{equation}
	
	From the construction of the random variables, we have that for each $i$ and $j$
	\begin{align*}
		\lim_{M\rightarrow\infty}w^j_{\text{SC},M, i}(\omega) =& w^j_i,\\
		\lim_{M\rightarrow\infty}x^j_{\text{SC}, M}(\omega) =& x^j,\\
		\lim_{M\rightarrow\infty}b_{\text{SC},M, i}(\omega) =& b_i.
	\end{align*}
	Therefore, these exists $M_{min}(\omega)$ such that for all $M\ge M_{min}(\omega)$ and all $\epsilon>0$, we have
	\begin{align*}
		\big|w^j_{\text{SC},M, i}(\omega)x^j_{\text{SC}, M}(\omega) - w^j_ix^j\big| <& \epsilon'\\
		\big|b_{\text{SC},M, i}(\omega) - b_i\big| <& \epsilon',
	\end{align*}
	where $\epsilon' = \frac{1}{n+1}\epsilon$. Use an argument of triangle inequality to show
	\begin{equation}
	\Big|\bm{w}^{\mathsf{T}}_{\text{SC},M, i}(\omega)\cdot\bm{x}_{\text{SC},M}(\omega) + b_{\text{SC},M,i}(\omega) -\bm{w}^{\mathsf{T}}_{ i}\bm{x} - b_i\Big| < \epsilon
	\end{equation}
	Since $\epsilon$ can be arbitrarily small, it implies
	\begin{equation}
	\lim_{M\rightarrow\infty}\bm{w}^{\mathsf{T}}_{\text{SC},M, i}(\omega)\cdot\bm{x}_{\text{SC},M}(\omega) + b_{\text{SC},M,i}(\omega) = \bm{w}^{\mathsf{T}}_{ i}\bm{x} + b_i.
	\end{equation}
	Since this is true for every $\omega\in\Omega$, we conclude that
	\begin{equation}
	P\big(\big\{\omega\in\Omega:\lim_{M\rightarrow\infty}\bm{w}^{\mathsf{T}}_{\text{SC},M, i}(\omega)\cdot\bm{x}_{\text{SC},M}(\omega) + b_{\text{SC},M,i}(\omega) = \bm{w}^{\mathsf{T}}_{ i}\bm{x} + b_i\big\}\big) = 1.
	\end{equation}
	In other words, we proved that as $M\rightarrow\infty$, the stochastic number $\bm{w}^{\mathsf{T}}_{\text{SC},M, i}(\omega)\cdot\bm{x}_{\text{SC},M}(\omega) + b_{\text{SC},M,i}(\omega)$ almost surely converges to $\bm{w}^{\mathsf{T}}_{ i}\bm{x} + b_i$.
\end{proof}

\begin{lemma}\label{lem:2}
	If the sigmodial function $\sigma(t)$ has bounded derivative, then the stochastic number $\sigma(\bm{w}^{\mathsf{T}}_{\text{SC},M, i}\bm{x}_{\text{SC},M} + b_{\text{SC},M,i})$ almost surely converge to the real value $\sigma(\bm{w}^{\mathsf{T}}_{i}\bm{x}+b_{i})$ as the bit-stream length $M\rightarrow \infty$, .
\end{lemma}
\begin{proof}
	We have the following inequalities:
	\begin{equation}
	\begin{split}
	&\big|\sigma(\bm{w}^{\mathsf{T}}_{\text{SC},M, i}\bm{x}_{\text{SC},M} + b_{\text{SC},M,i})-\sigma(\bm{w}^{\mathsf{T}}_{i}\bm{x}+b_{i})\big|\\
	&\le \max_t \big|\sigma'(t)\cdot |\bm{w}^{\mathsf{T}}_{\text{SC},M, i}\bm{x}_{\text{SC},M} + b_{\text{SC},M,i}-\bm{w}^{\mathsf{T}}_{i}\bm{x}-b_{i}| \big|\\
	&\le \big(\max_t \big|\sigma'(t) \big|\big)\cdot \big| \bm{w}^{\mathsf{T}}_{\text{SC},M, i}\bm{x}_{\text{SC},M} + b_{\text{SC},M,i}-\bm{w}^{\mathsf{T}}_{i}\bm{x}-b_{i}\big|
	\end{split}
	\end{equation}
	For the currently utilized activation functions, including sigmoid, tanh (hyperbolic tangent), ReLU functions, there is an upper bound on the derivatives. The maximum absolute value of the derivatives is often 1. Then, from the above Lemma 1 about the almost sure convergence of $\bm{w}^{\mathsf{T}}_{\text{SC},M, i}\bm{x}_{\text{SC},M} + b_{\text{SC},M,i}$ to $\bm{w}^{\mathsf{T}}_{i}\bm{x}+b_{i}$, we arrive at the almost sure convergence of $\sigma(\bm{w}^{\mathsf{T}}_{\text{SC},M, i}\bm{x}_{\text{SC},M} + b_{\text{SC},M,i})$. 
\end{proof}

Based on the above lemmas and the original universal approximation theorem, we arrive at the following universal approximation theorem for SCNNs.

\newtheorem{theorem}{Theorem}[section]
\begin{theorem}\label{thm:1}
	(Universal Approximation Theorem for SCNNs). For any continuous function $f(\bm{x})$ defined on $I_{n}$ and any $\epsilon>0$, we define an event that there exists an SCNN function $G_{\text{SC},M}(\bm{x}_{\text{SC},M})$ in the form of Eqn. (4) that satisfies 
	\begin{equation}
	\lim_{M\rightarrow\infty}\left|G_{\text{SC},M}(\bm{x}_{\text{SC},M})-f(\bm{x})\right|<\epsilon.
	\end{equation}
	This event is satisfied almost surely (with probability 1).
\end{theorem}
\begin{proof}
	From the universal approximation theorem stated in \cite{cybenko1989approximation}, we know that there exists a function $G(x)$ representing a deterministic neural network such that $|G(x) - f(x)|<\epsilon/2$ for all $x\in I_n$. For each positive integer $M$ define $G_{\text{SC}, M}(x)$ as the SCNN function obtained by replacing each parameter of $G(x)$ with its $M$-bit stochastic representation. Then we have
	\begin{equation}
	\begin{split}
	\left|G_{\text{SC},M}(\bm{x}_{\text{SC},M})-f(\bm{x})\right|&=\left|G_{\text{SC},M}(\bm{x}_{\text{SC},M})-G(\bm{x})+G(\bm{x})-f(\bm{x})\right|\\
	& \leq \left|G_{\text{SC},M}(\bm{x}_{\text{SC},M})-G(\bm{x})\right|+\left|G(\bm{x})-f(\bm{x})\right|
	\end{split}
	\end{equation}
	Applying Lemma \ref{lem:1} and \ref{lem:2}, we can bound the first term as 
	\begin{equation}
	\begin{split}
	\left|G_{\text{SC},M}(\bm{x}_{\text{SC},M})-G(\bm{x})\right|
	&=\bigg|\sum^N_{i=1}\alpha_{i}\sigma(\bm{w}^{\mathsf{T}}_{\text{SC},M, i}\bm{x}_{\text{SC},M} + b_{\text{SC},M,i})-\sum^N_{i=1}\alpha_{i}\sigma(\bm{w}^{\mathsf{T}}_{i}\bm{x}+b_{i})\bigg|\\
	&\le\sum^N_{i=1}\bigg|\alpha_{i}\Big[\sigma(\bm{w}^{\mathsf{T}}_{\text{SC},M, i}\bm{x}_{\text{SC},M} + b_{\text{SC},M,i})-\sigma(\bm{w}^{\mathsf{T}}_{i}\bm{x}+b_{i})\Big]\bigg|\\
	&\le \sum^N_{i=1}\big|\alpha_{i}\big|\cdot\big|\sigma(\bm{w}^{\mathsf{T}}_{\text{SC},M, i}\bm{x}_{\text{SC},M} + b_{\text{SC},M,i})-\sigma(\bm{w}^{\mathsf{T}}_{i}\bm{x}+b_{i})\big|
	\end{split}
	\end{equation}
	where $N$ is the size of the hidden layer in the neural network represented by $G(\bm{x})$, and $\alpha_i$ is the $i$-th weight in the output layer.
	
	
	Deriving from Lemma 2, we know that for $\displaystyle\frac{\epsilon}{2\sum_{i=1}^N\alpha_i}>0$, with probability 1 there exists $M_{min}$ such that \begin{equation}
	\big|\sigma(\bm{w}^{\mathsf{T}}_{\text{SC},M, i}\bm{x}_{\text{SC},M} + b_{\text{SC},M,i})-\sigma(\bm{w}^{\mathsf{T}}_{i}\bm{x}+b_{i})\big|<\frac{\epsilon}{2\sum_{i=1}^N\alpha_i}
	\end{equation}
	for $M\ge M_{min}$. Incorporating into Eqn. (13) we have $\displaystyle\left|G_{\text{SC},M}(\bm{x}_{\text{SC},M})-G(\bm{x})\right|<\frac{\epsilon}{2}$. Further incorporating into Eqn. (12) we have $\left|G_{\text{SC},M}(\bm{x}_{\text{SC},M})-f(\bm{x})\right|<\epsilon$ for $M\ge M_{min}$. Thereby we have formally proved that universal approximation theorem holds with probability 1 for SCNNs.
\end{proof}

Besides the universal approximation property, it is also critical to derive an appropriate bound for bit length $M$ in order to provide insights for the actual neural network implementations. 
The next theorem gives an explicit bound on the bit length for close approximation with high probability.
\begin{theorem}
	For the SCNN function $G_{\text{SC}, M}$ in Theorem \ref{thm:1}, let $M$ be any integer that satisfies
	\begin{equation}
	M > \frac{(n+1)^2\cdot N^2}{\epsilon^2\delta}.
	\end{equation}
	Then with probability at least $1-\delta$, $\left|G_{\text{SC},M}(\bm{x}_{\text{SC},M})-f(\bm{x})\right|<\epsilon$ holds for all $x\in I_n$.
\end{theorem}
\begin{proof}
	
	Different from the above proof based on the strong law of large numbers (almost sure convergence), deriving bounds is more related to the weak law (convergence in probability). As the former case will naturally ensure the latter, we have the following convergence in probability property: For any $\epsilon, \delta>0$, there exists $M^\delta_{min}$, such that for any $M\ge M^\delta_{min}$, we have \begin{equation}
	Pr\Big\{\left|G_{\text{SC},M}(\bm{x}_{\text{SC},M})-f(\bm{x})\right|<\epsilon\Big\}>1-\delta
	\end{equation}
	
	Based on a reverse order of the above proof of universal approximation, the above inequality is satisfied when we have \begin{equation}
	Pr\Big\{\left|G_{\text{SC},M}(\bm{x}_{\text{SC},M})-G(\bm{x})\right|<\frac{\epsilon}{2}\Big\}>1-\delta
	\end{equation} Furthermore, the above inequality is satisfied when we have
	\begin{equation}
	Pr\bigg\{\left|w^j_{\text{SC},M,i}x^j_{\text{SC},M}-w^j_{i}x^j\right|<\frac{\epsilon}{2(n+1)\cdot \sum_i \alpha_i}\bigg\}>1-\delta
	\end{equation}
	
	As each bit in stochastic number $w^j_{\text{SC},M,i}x^j_{\text{SC},M}$ satisfies a binary distribution with expectation $w^j_{i}x^j$, the maximum variance is $\displaystyle\frac{1}{4}$. Due to i.i.d. property, the maximum variance ($\sigma^2$) of $w^j_{\text{SC},M,i}x^j_{\text{SC},M}$ is $\displaystyle\frac{1}{4M}$. According to the Chebyshev's inequality $\displaystyle Pr\big(\|X-\mu\|\ge k\sigma\big)\le \frac{1}{k^2}$, we let $\displaystyle \frac{1}{k^2}=\delta$ and obtain 
	\begin{equation}
	\frac{1}{2\sqrt{\delta M}}=\frac{\epsilon}{2(n+1)\cdot \sum_i \alpha_i}
	\end{equation}
	Then we derive an upper bound of $M^\delta_{min}$ as 
	\begin{equation}
	M_{min}^\delta\le\frac{(n+1)^2\cdot \big(\sum_i\alpha_i\big)^2}{\epsilon^2\delta}\le \frac{(n+1)^2\cdot N^2}{\epsilon^2\delta}
	\end{equation}
\end{proof}

\subsection{Universal Approximation of BNNs and Equivalence between SCNNs and BNNs}

In this section we start from the formal definition of BNNs of interests and then state the universal approximation property. Similar to the definition of SCNNs in Section 3, here we focus on an "ideal" BNN that is independent of actual BNN implementations. An illustration is shown in Figure \ref{fig_BNN}.

\begin{wrapfigure}{r}{0.5\textwidth}
	\begin{center}
		\includegraphics[width = 0.42\textwidth]{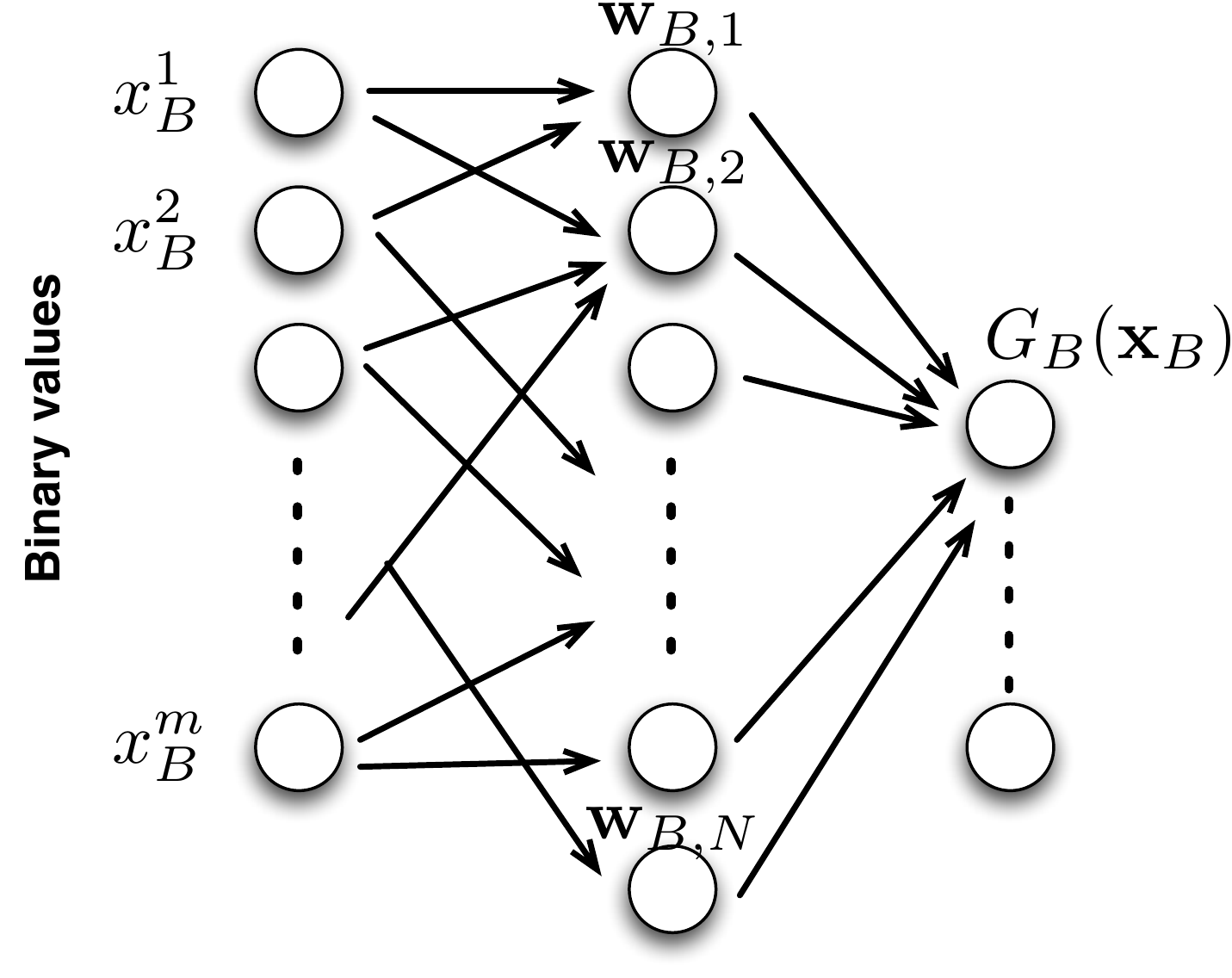}
		\caption{The structure of BNN of interest.}
		\label{fig_BNN}
	\end{center}
	\vspace{-0.3cm}
\end{wrapfigure}

\begin{definition}
	A BNN of interest is defined as a function $G_{\text{B}}(\bm{x}_{\text{B}})$, satisfying:
	\begin{equation}
	G_{\text{B}}(\bm{x}_{\text{B}})=\sum^N_{i=1}\alpha_{i}\sigma(\bm{w}^{\mathsf{T}}_{\text{B}, i}\bm{x}_{\text{B}}+b_{\text{B},i})
	\end{equation}
	where the input vector $\bm{x}_{\text{B}}$ and weight vector $\bm{w}_{\text{B},i}$ for each $i$ represent vectors of binary values. Let $m$ denote the dimensionality in these two vectors (dimension of inputs). $b_{\text{B},i}$ is a binary bias value. 
	The computation in $\bm{w}^{\mathsf{T}}_{\text{B}, i}\bm{x}_{\text{B}}+b_{\text{B},i}$ follows the BNN rules as described in Section 2.2.
	Similar to SCNNs, we also consider here accurate activation and output layer calculation. This is reasonable and also applied in BNN deployments because the output layer size is typically very small.
\end{definition}

\textbf{The Equivalence of SCNNs and BNNs:} The BNNs can be transformed into SCNNs, and vice versa. We illustrate the former case as an example. Let $M$ denote the length of stochastic number and the number of inputs in SCNN becomes $\displaystyle n=\frac{m}{M}$. Then the first input stochastic number $x_{\text{SC}}^1=\bm{x}_{\text{B}}[1:M]$ (i.e., the first $M$ bits in $\bm{x}_{\text{B}}$), the second input stochastic number $x_{\text{SC}}^2=\bm{x}_{\text{B}}[M+1:2M]$, and so on. This also applies to the weight stochastic numbers. The bias stochastic number $b_{\text{SC},i}$ can be a sign extension of $b_{\text{B},i}$. In this way the BNN is transformed into SCNN described in Definition 2. The transformation from SCNN to BNN is similar.

Because of the universal approximation property of SCNNs and the equivalence of BNNs, we arrive at the universal approximation for BNNs as well.

\begin{theorem}
	(Universal Approximation Theorem for BNNs). For any continuous function $f(\bm{x})$ defined on $I_{n}$, $\epsilon>0$, we define an event that there exists an BNN function $G_{\text{B}}(\bm{x}_{\text{B}})$ in the form of Eqn. (21) that satisfies 
	\begin{equation}
	\lim_{m\rightarrow\infty}\left|G_{\text{B}}(\bm{x}_{\text{B}})-f(\bm{x})\right|<\epsilon.
	\end{equation}
	This event is satisfied almost surely (with probability 1).
\end{theorem}
\begin{proof}
	Apply Theorem \ref{thm:1} to obtain a close approximation of $f(x)$ with SCNN functions, then build a BNN function that closely approximations the SCNN function.
\end{proof}

The equivalence in SCNNs and BNNs also leads to the same bound, defined as the total number of input bits $m=n\cdot M$ required to achieve universal approximation. The reasoning is using proof by contradiction. Suppose that SCNNs have a lower bound, i.e., $n\cdot M_{min}<m_{min}$. Then there exists an SCNN with $n$ inputs each with $M_{min}$ bits satisfying the universal approximation property. From the above equivalance analysis we can construct a BNN with $M_{min}\cdot n$ input bits that also achieves such property, which is smaller and thus in contradiction with the bound $m_{min}$. And vice versa.

\section{Energy Complexity and Hardware Design Implications}
\subsection{Energy Complexity Analysis}
The \emph{energy complexity}, as defined and described in \cite{martin2001towards,khude2005time}, specifies the asymptotic energy consumption with the growth of neural network size. It can be perceived as a multiplication of the time complexity and parallelism degree, and therefore is important for hardware implementations and evaluations. As an example, when the input size (number of bits) is $n$, a ripple carry adder has an energy complexity of O($n$) whereas a multiplier has energy complexity of O($n^2$). On the other hand, both of their time complexity is O($n$). The reason is because the ripple carry adder is a sequential computation whereas the multiplier is a parallel computation.

Next we provide an analysis on the energy complexity of the key calculation in $\bm{w}^{\mathsf{T}}_{\text{SC},M, i}\bm{x}_{\text{SC},M}+b_{\text{SC},M,i}$ in SCNNs and  $\bm{w^{\mathsf{T}}_{\text{B}, i}x_{\text{B}}}+b_{\text{B},i}$ in BNNs. From the equivalence analysis in Section 4.2, we have $m=n\cdot M$ and $M\ge M_{min}$ for satisfying the universal approximation property. According to the hardware implementation details in Section 2, the multiplication of two bits has energy complexity of O(1), then the multiplication of two stochastic numbers has energy complexity of O($M$). The addition of a set of $n$ stochastic numbers has energy complexity of O($nM$) using simple calculation units like multiplexers or energy complexity O($n\log n \cdot M$) using more accurate accumulation units like the approximate parallel counter (APC) \cite{kyounghoon2015approximate}. As a result, the overall energy complexity in $\bm{w}^{\mathsf{T}}_{\text{SC},M, i}\bm{x}_{\text{SC},M}+b_{\text{SC},M,i}$ is O($nM$) (for less accurate results) or O($n\log n \cdot M$) (for more accurate results). For the whole layer with $N$ neurons, the overall energy complexity is $n\cdot M\cdot N$ or $n\log n\cdot M\cdot N$. The energy complexity for BNNs with $m=n\cdot M$ is the same due to the equivalence.

\subsection{Hardware Design Implications}

Despite the same energy complexity, the actual hardware implementations of SCNNs and BNNs are different. 
As discussed before, SCNNs "stretch" in the temporal domain whereas BNNs span in the spatial domain. This is in fact the most important advantage of SCNNs. For BNN actual implementations, there is often an imbalance between the input I/O size and the computation requirement. The total computation requirement (please refer to the energy complexity discussion) is low, but the input requirement is huge even compared with conventional neural networks. This makes actual BNN implementations I/O bound systems, as in actual hardware tapeouts the I/O clock frequency is much lower compared with the computation clock frequency. In other words, the advantage of low and simple computation in BNNs is often not fully exploited in actual deployments \cite{zhao2017accelerating,umuroglu2017finn}. This limitation can be effectively mitigated by SCNNs, because the spatial requirement is effectively traded-off with the temporal requirement. In this aspect SCNNs can use lower I/O account and thereby more effective usage of hardware computation and memory storage resources compared with BNN counterparts, thereby becoming more suitable for hardware implementations.

On the other hand, BNNs are more heavily optimized in literature compared with SCNNs. Especially, many research work \cite{courbariaux2015binaryconnect,hubara2016quantized} are dedicated for effective training methods for BNNs making efficient usage of randomization techniques. On the other hand, the research on SCNNs are mainly from the hardware aspect \cite{ren2017sc,yu2017accurate,li2017structural}. For training these work use a straightforward way of transforming directly (every input and weight) from conventional neural networks to stochastic numbers. As a result, it will be effective to take advantage of the training methods for BNNs, transform into SCNNs that are more suitable for hardware implementations using the method described in Section 4.2. In this way, we can effectively exploit the advantage while hiding weakness in both SCNNs and BNNs.

\section{Conclusion}

SCNNs and BNNs are low-complexity variants of deep neural networks that are particularly suitable for hardware implementations.
In this paper, we conduct theoretical analysis and comparison between SCNNs and BNNs in terms of universal approximation property, energy complexity, and suitability for hardware implementations. 
More specifically, we prove that the "ideal" SCNNs and BNNs satisfy the universal approximation property with probability 1. The proof is conducted by first proving the property for SCNNs from the strong law of large numbers, and then using SCNNs as a "bridge" to prove for BNNs. Based on the universal approximation property, we further prove that SCNNs and BNNs exhibit the same energy complexity. In other words, they have the same asymptotic energy consumption with the growing of network size. We also provide a detailed analysis of the pros and cons of SCNNs and BNNs for hardware implementations and present a way of effectively exploiting the advantage of each type while hiding the weakness.

\bibliography{./ref}
\bibliographystyle{splncs}

\end{document}